\documentclass{article}
\usepackage{ijcai26}
\usepackage{times}
\usepackage{soul}
\usepackage{url}
\usepackage[hidelinks]{hyperref}
\usepackage[utf8]{inputenc}
\usepackage[small]{caption}
\usepackage{graphicx}
\usepackage{amsmath}
\usepackage{amssymb}
\usepackage{amsthm}
\newtheorem{theorem}{Theorem}

\usepackage{booktabs}
\usepackage{pifont}
\newcommand{\cmark}{\ding{51}}
\newcommand{\xmark}{\ding{55}}
\usepackage{algorithm}
\usepackage{algorithmic}
\usepackage{lineno}
\usepackage{xcolor}
\usepackage{multirow}
\usepackage{colortbl}
\usepackage{tikz}
\usepackage{pgfplots}
\pgfplotsset{compat=1.18}
\usetikzlibrary{arrows.meta,positioning,shapes.geometric,fit,backgrounds,calc}

\definecolor{symgreen}{RGB}{34,139,34}
\definecolor{symred}{RGB}{180,40,40}
\definecolor{symblue}{RGB}{30,90,180}
\definecolor{symgray}{RGB}{120,120,120}
\definecolor{placeholder}{RGB}{180,60,60}

\newcommand{\SPDE}{0/2}   \newcommand{\SPDM}{0/2}   \newcommand{\SPDH}{0/2}   \newcommand{\SPDT}{0/6~(0\%)}
\newcommand{\SPCE}{0/2}   \newcommand{\SPCM}{0/2}   \newcommand{\SPCH}{0/2}   \newcommand{\SPCT}{0/6~(0\%)}
\newcommand{\SPSRE}{1/2}  \newcommand{\SPSRM}{1/2}  \newcommand{\SPSRH}{0/2}  \newcommand{\SPSRTtot}{2/6~(33\%)}
\newcommand{\SPLLE}{0/2}  \newcommand{\SPLLM}{0/2}  \newcommand{\SPLLH}{0/2}  \newcommand{\SPLLT}{0/6~(0\%)}
\newcommand{\SPSSE}{2/2}  \newcommand{\SPSSM}{2/2}  \newcommand{\SPSSH}{1/2}  \newcommand{\SPSST}{5/6~(83\%)}  
\newcommand{\SPSGSE}{2/2} \newcommand{\SPSGSM}{1/2} \newcommand{\SPSGSH}{1/2} \newcommand{\SPSGST}{4/6~(67\%)} 
\newcommand{\MWPDE}{3/3}   \newcommand{\MWPDM}{3/3}   \newcommand{\MWPDH}{1/2}   \newcommand{\MWPDT}{7/8~(88\%)}
\newcommand{\MWPCE}{3/3}   \newcommand{\MWPCM}{3/3}   \newcommand{\MWPCH}{2/2}   \newcommand{\MWPCT}{8/8~(100\%)}
\newcommand{\MWPSSE}{3/3}  \newcommand{\MWPSSM}{3/3}  \newcommand{\MWPSSH}{2/2}  \newcommand{\MWPSST}{8/8~(100\%)}
\newcommand{\MWPSGSE}{3/3} \newcommand{\MWPSGSM}{3/3} \newcommand{\MWPSGSH}{2/2} \newcommand{\MWPSGST}{8/8~(100\%)}
\newcommand{\FINDE}{1/2}   \newcommand{\FINDM}{1/2}   \newcommand{\FINDH}{2/2}   \newcommand{\FINDT}{4/6~(67\%)}
\newcommand{\FINCE}{2/2}   \newcommand{\FINCM}{1/2}   \newcommand{\FINCH}{2/2}   \newcommand{\FINCT}{5/6~(83\%)}
\newcommand{\FINSSE}{2/2}  \newcommand{\FINSSM}{2/2}  \newcommand{\FINSSH}{2/2}  \newcommand{\FINSST}{6/6~(100\%)}  
\newcommand{\FINSGSE}{2/2} \newcommand{\FINSGSM}{2/2} \newcommand{\FINSGSH}{2/2} \newcommand{\FINSGST}{6/6~(100\%)} 
\newcommand{\ZLDT}{0/35}            
\newcommand{\ZLCT}{0/35}            
\newcommand{\ZLSST}{28/35}          
\newcommand{\ZLSGST}{34/35}         
\newcommand{\ZLDpct}{0\%}
\newcommand{\ZLCpct}{0\%}
\newcommand{\ZLSSpct}{80\%}
\newcommand{\ZLSGSpct}{97\%}

\newcommand{\LSATDpct}{100\%}   \newcommand{\LSATDtot}{15/15}
\newcommand{\LSATCpct}{87\%}    \newcommand{\LSATCtot}{13/15}
\newcommand{\LSATSSpct}{100\%}  \newcommand{\LSATSStot}{15/15}
\newcommand{\LSATSGSpct}{100\%} \newcommand{\LSATSGStot}{15/15}
\newcommand{\AQDpct}{86\%}      \newcommand{\AQDtot}{30/35}
\newcommand{\AQCpct}{89\%}      \newcommand{\AQCtot}{31/35}
\newcommand{\AQSSpct}{66\%}     \newcommand{\AQSStot}{23/35}
\newcommand{\AQSGSpct}{89\%}    \newcommand{\AQSGStot}{31/35}

\title{SymStep: Symbolic Step Verification for Logical Reasoning}

\author{
    Aida Usmanova$^1$
    \and
    Rui Gao$^2$\and
    Dilshod Azizov$^2$\and
    Ricardo Usbeck$^1$\And
    Zangir Iklassov$^2$\\
    \affiliations
    $^1$Leuphana University of Lüneburg\\
    $^2$Mohamed bin Zayed University of Artificial Intelligence\\
    \emails
    \{aida.usmanova, ricardo.usbeck\}@leuphana.de,
    \{rui.gao, dilshod.azizov, zangir.iklassov\}@mbzuai.ac.ae
}

\begin{document}
\maketitle

\begin{abstract}
Chain-of-thought (CoT) prompting can fail severely on constraint-dense logical reasoning tasks, where unverified errors accumulate silently across steps.
We introduce \textbf{SymStep}: an LLM makes one atomic claim at a time (\texttt{DEDUCE: Alice, pet, Cat}), then a lightweight constraint propagator checks the claim for consistency with prior accepted deductions, rejects contradictions, and cascades implied facts automatically.
\textbf{SymStep+G} additionally provides an MRV guidance after each accepted step, directing the LLM toward the most constrained unresolved variable.
On a 35-puzzle retained subset of \textbf{ZebraLogicBench}---benchmark of 1,000 Einstein-style logic puzzles---Direct and CoT both achieve \textbf{0\%}, while SymStep+G reaches \textbf{97\%}.
On \textbf{AR-LSAT} analytical reasoning problems, SymStep achieves \textbf{100\%} vs.\ CoT's 87\%.
On \textbf{LGP-14}, SymStep+G achieves \textbf{100\%} vs.\ \textbf{0\%} for CoT and Logic-LM, the strongest prior symbolic$+$LLM baseline we compare against.
Ablation studies reveal that MRV guidance is a key mechanism for reducing directionless cycling, while consistency checking provides a safety net against explicit contradictions.
Across six benchmarks spanning five task domains, SymStep variants match or exceed every baseline on constraint-dense and arithmetic tasks. Experiments on AQUA-RAT algebra confirm the advantage is constraint-density-specific.
\footnote{Code and LGP-20 benchmark available at: \\ \url{https://github.com/Zangir/SymStep}}
\end{abstract}

\section{Introduction}
\label{sec:intro}

Multi-step logical reasoning remains an open problem for Large Language Models (LLMs)~\cite{beygi2022logical,patel2024multi}. Current reasoning approaches perform poorly on relevant benchmarks:
On ZebraLogicBench~\cite{lin2025zebralogic}, a benchmark of 1,000 Einstein-style logic puzzles, direct prompting and chain-of-thought (CoT) achieve 0\%.
On our own logical reasoning benchmark, CoT and Logic-LM, the strongest prior symbolic$+$LLM method, also achieve 0\% (Table~\ref{tab:motivation}). Two failure modes emerge: (1) \emph{unverified step accumulation}: models produce steps during reasoning, each unverified step can introduce an error, and subsequent steps compound that error without detection; (2) \emph{deductive deadlock}: failure occurs even when deductions are individually correct, but without external direction, the LLM cycles without progress.

\begin{figure*}
    \centering
    \includegraphics[width=\linewidth]{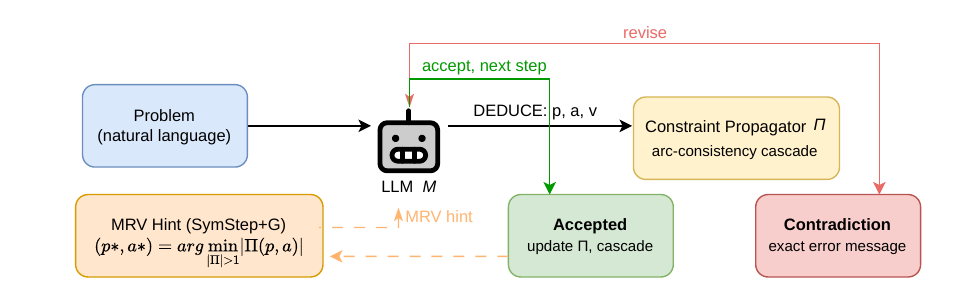}
    \caption{SymStep pipeline (\textbf{solid}: both variants; \textbf{orange dashed}: SymStep+G only).
  The LLM generates one \textsc{Deduce} assertion per turn; the constraint propagator
  checks consistency and cascades arc-consistency automatically.
  \textbf{Green}: the deduction is accepted, $\Pi$ is updated, and the LLM continues.
  \textbf{Red}: the exact contradiction is reported, and the LLM revises.
  \textbf{Orange dashed} (SymStep+G only): the MRV hint directs the LLM toward the
  most constrained undetermined variable $(p^*,a^*)$.
}
    \label{fig:overview}
\end{figure*}

\begin{table}[t]
\centering\small
\setlength{\tabcolsep}{5pt}
\begin{tabular}{lcc}
\toprule
\textbf{Method} & \textbf{LGP-14} & \textbf{LGP-10}\\
\midrule
Direct            & 14\%  & 20\%  \\
\rowcolor{symred!8}
CoT               & \textbf{0\%}   & \textbf{0\%}   \\
Self-Refine       & 36\%  & 30\%  \\
Logic-LM          &  0\%  &  --  \\
SymStep           & 86\%  & 80\%  \\
\rowcolor{symgreen!10}
\textbf{SymStep+G}& \textbf{100\%} & \textbf{100\%}  \\
\bottomrule
\end{tabular}
\caption{Performance comparison on LGP-14 and LGP-10 (10-puzzle subset of LGP-14). CoT and Logic-LM achieve 0\% on logical reasoning.
SymStep+G achieves the best accuracy on both puzzle sets.
All methods are evaluated on Claude Haiku.
}
\label{tab:motivation}
\end{table}

Logic-LM~\cite{pan2023logiclm} and LINC~\cite{olausson2023linc} translate
the \emph{entire problem} into formal logic in a single step before invoking a
solver; if that translation is wrong anywhere, the whole attempt fails with
no opportunity for correction. Although DVCOTR~\cite{ling2023deductive} decomposes reasoning into steps and verifies them, it uses the LLM itself as the verifier, leaving LLM errors and biases undetected. SymStep takes a different path. First, every verification decision is
\emph{deterministic}: a constraint propagator, not an LLM, checks each claim for consistency before accepting it. Second, the propagator's
arc-consistency cascade automatically derives implied eliminations that no
single clue reveals in isolation, performing the bookkeeping that the LLM cannot reliably perform on its own. Third, MRV-based guidance tells the LLM which variable to tackle next, resolving the deductive deadlock challenge.

\paragraph{Contributions.}
\begin{enumerate}
  \item \textbf{SymStep and SymStep+G.}
        SymStep structures LLM output into one atomic \textsc{Deduce} assertion per turn. A constraint propagator then verifies it and cascades implied eliminations, or feeds back the exact contradiction if it is wrong.
        SymStep+G additionally provides an MRV hint containing the unresolved variable with the fewest remaining candidates after each verified step, eliminating deductive deadlocks before they arise. Our algorithm requires no formal logic training, only regex-parseable output.

  \item \textbf{Ablation insight: guidance is the primary driver.}
        A controlled ablation (Table~\ref{tab:ablation_comp}) reveals that MRV guidance alone achieves 100\% with zero variance, while verification alone reaches only 72\%.
        This result shows that the key bottleneck is not incorrect deductions but directionless search: once pointed to the most constrained variable, the LLM makes only correct deductions, and the safety-net verifier is rarely triggered.

  \item \textbf{Six-domain evaluations.}
        We test on ZebraLogicBench (SymStep+G 97\% vs.\ CoT 0\%), AR-LSAT analytical reasoning (SymStep 100\% vs.\ CoT 87\%), LGP-14 (100\%), and three additional domains (scheduling, arithmetic, and financial reasoning).
        A domain-boundary experiment on AQUA-RAT algebra (Appendix~\ref{app:aquarat}) confirms that the advantage is constraint-density-specific: where constraint propagation is not needed, SymStep+G matches CoT with no degradation.

\end{enumerate}

\section{Problem Formulation}
\label{sec:problem}

\subsection{Logical Grid Puzzles}

Logical grid puzzles are the clearest way to test step-level deductive reasoning because every fact must follow explicit constraints; there is no world knowledge to fall back on, and any undetected error propagates through all subsequent deductions.
In a scenario such as \emph{``Alice, Bob, and Carol each have a
different pet (cat, dog, fish) and live in differently colored houses
(red, blue, green). Alice does not have a dog. Bob lives in the blue
house. \ldots''}, the goal is to determine who has what, using only the
given clues. Every attribute assignment must be derived step by step from
the constraints.

Formally, a problem has $n$ \emph{entities} (e.g., people) and $k$ \emph{attributes} (e.g., pet, job), each with a finite domain of values. Every attribute is \emph{bijective}: each value is assigned to exactly one entity. The task is to uniquely determine all entity--attribute--value assignments from the given clues.

\textbf{Solution uniqueness.} We require each puzzle to have \emph{exactly
one} valid assignment. This uniqueness is verified by running the constraint
propagator on the stored solution and confirming that it reaches a fully
determined state with no alternatives, a check that also caught
\emph{two} originally included puzzles (E3-Sport and M2-Nationals) that
each admitted two valid solutions; automated uniqueness checks are mandatory in deductive benchmarks.
(See \S\ref{sec:lgp14}.)

\subsection{Generalization to Other Reasoning Domains}

While LGPs are our primary benchmark, SymStep's architecture is
domain-agnostic: it applies to \emph{any task that can be decomposed into
a sequence of atomic constraint assertions with checkable consistency}.
We demonstrate this on three additional domains:

\textbf{Scheduling puzzles (SP).}
Constraint satisfaction problems in a scheduling context, e.g.,
assigning presenters to conference time slots given ordering and
compatibility constraints. SP share the bijective structure of LGPs
but use different domain vocabulary (time slots, rooms, roles) from the LGP distribution, enabling cross-domain transfer tests.

\textbf{Math word problems (MWP) and financial reasoning (FIN).}
Multi-step numeric reasoning (e.g., deriving employee salaries, compound
interest, or tax liabilities from arithmetic relationships).
Each \textsc{Deduce} step asserts a variable's numeric value. Verifier
checks arithmetic consistency and cascades any directly implied values.
The MRV hint identifies the next \emph{computable} unknown: variable
whose constraint becomes evaluable given current knowledge.

All extensions reuse the same \textsc{Deduce} protocol and outer
SymStep loop, only the symbolic verifier is adapted per domain.

\subsection{Why Step-Level Verification?}
\label{sec:why_stepwise}

Consider a small puzzle from E2-Jobs (Adam/Beth/Chris,
drinks: Tea/Coffee/Water, jobs: Doctor/Teacher/Engineer):

\begin{quote}\small
\texttt{Adam drinks tea (clue 1). The doctor drinks coffee (clue 2),}\\
\texttt{so Adam is not the doctor. Beth is not the engineer (clue 3).}\\
\texttt{Chris is not the doctor (clue 4). So the doctor must be Beth.}\\
\texttt{Adam is not the teacher (clue 5), so Adam is the engineer.}\\
\texttt{Chris is the teacher. Beth drinks coffee.}\\
\texttt{Therefore: Adam=Tea+\textbf{Teacher}, Beth=Coffee+Doctor, Chris=Water+\textbf{Engineer}.}
\end{quote}

Although the reasoning chain correctly establishes ``Adam is the engineer'' in step 4, it contradicts that conclusion in the final answer (``Adam=Teacher'').
The parser receives ``Adam=Teacher'' as the final answer and marks the puzzle
incorrect. This is \emph{unverified step accumulation}: each new sentence updates
the narrative without checking it against what has already been
established. The longer the chain, the greater the chance that a contradiction will
go unnoticed.

SymStep prevents this by modeling reasoning as a sequence of atomic constraint
assertions $d_t = (\sigma(p,a)=v)$ or $d_t = (\sigma(p,a)\neq v)$.
A deduction is \emph{consistent} if it does not contradict any previously
accepted deduction. The propagator enforces this property at every step,
ensuring that no inconsistent fact enters the context.

\section{SymStep}
\label{sec:method}

\subsection{The \textsc{Deduce} Protocol}

SymStep structures LLM--verifier interaction through a minimal
\textbf{\textsc{Deduce} protocol}: the LLM produces exactly one
constraint assertion per step:
\begin{align*}
&\texttt{DEDUCE: \textit{Entity}, \textit{attribute}, \textit{Value}}\\
&\texttt{DEDUCE: \textit{Entity}, \textit{attribute}, NOT \textit{Value}}
\end{align*}
This format is parseable by regex, requires no formal logic training,
and is far less demanding than generating a complete Prolog program or
FOL formula. The LLM signals task completion with
\texttt{CONCLUDE: done}.

\subsection{Constraint Propagator}
\label{sec:propagator}

The symbolic backbone of SymStep is a \textbf{constraint propagator}
that maintains a possibility matrix:
$\Pi(p,a) \subseteq V_a$: a set of values still consistent with
all accepted deductions for entity $p$, attribute $a$.
Initially, $\Pi(p,a) = V_a$ for all $p, a$.

\textbf{Intuition.} $\Pi$ is similar to a grid where each cell lists
which values a person might have for a given attribute. When the LLM
deduces ``Alice has Red,'' the propagator (i) crosses Red off the grid
for every other person (arc-consistency), and (ii) checks
whether any cell has now shrunk to a single option. If so, that forced
assignment is propagated automatically, potentially triggering further eliminations without requiring an LLM call.

\paragraph{Positive update.}
On $\sigma(p,a)=v$: if $v \notin \Pi(p,a)$, return
\textsc{contradiction}. Otherwise, set $\Pi(p,a) \leftarrow \{v\}$
and remove $v$ from $\Pi(p',a)$ for all $p'\neq p$; if any such set
becomes empty, return \textsc{contradiction}.

\paragraph{Negative update.}
On $\sigma(p,a)\neq v$: discard $v$ from $\Pi(p,a)$; if
$\Pi(p,a) = \emptyset$, return \textsc{contradiction}.

\paragraph{Unit propagation (arc-consistency cascade).}
After each update, if $|\Pi(p,a)|=1$ for any undetermined $(p,a)$,
that singleton value is propagated as a positive deduction,
potentially triggering further eliminations.
This arc-consistency step may derive eliminations not derivable from
any single clue in isolation---the propagator ``does the bookkeeping''
that the LLM could miss.

The propagator returns one of: \textsc{ok}($msg$),
\textsc{contradiction}($msg$), or \textsc{solved}.
Contradiction messages identify the specific inconsistency, e.g.,\ \emph{``assigning
Cat to Alice leaves Bob.pet with no valid values.''} The propagator
runs in $O(n^2 k)$ time per call, which is faster than LLM inference.
Note that the propagator guarantees \emph{consistency} (no contradiction with prior accepted claims) not \emph{entailment}: a deduction consistent with the current state but not implied by the clues can still be accepted; the LLM bears responsibility for clue-grounded claims (formal analysis in Appendix~\ref{app:proofs}). The propagator guarantees consistency with prior accepted deductions, not logical entailment from the clues; a fabricated but locally consistent claim can still be accepted.

\subsection{SymStep: Verified Step-by-Step Reasoning}
\label{sec:symstep}

Figure~\ref{fig:overview} illustrates the full pipeline, and
Appendix Algorithm~\ref{alg:symstep} gives the pseudocode.
On each turn, the LLM receives the full interaction history $h$
and produces a single \textsc{Deduce} assertion.
The propagator checks it, and feedback (confirmation or contradiction)
is appended to $h$ before the next LLM call.
The LLM never sees an unverified fact accepted into its context, every piece of confirmed information is consistent.

\subsection{SymStep+G: MRV-Guided Reasoning}
\label{sec:symstepg}

After a verified step, SymStep+G appends an MRV hint computed from the
updated propagator state:
\begin{align}
  \textsc{MrvHint}(\Pi) &= \bigl(p^*,\, a^*,\, \Pi(p^*,a^*)\bigr),\nonumber\\
  (p^*,a^*) &= \arg\min_{\substack{p \in P,\; a \in \mathcal{A}\\|\Pi(p,a)|>1}} |\Pi(p,a)|
\end{align}
The hint is delivered in natural language, e.g.:
\begin{quote}\small
\textit{[Hint] Blake's drink must be one of: \{Coffee, Juice\}}
\end{quote}
The hint lists remaining candidates, not the answer.
It is derived entirely from propagated arc-consistency: information
the LLM cannot efficiently compute without external help.
The design mirrors the MRV heuristic in classical constraint
satisfaction~\cite{hoffmann2001ff}: choosing the most constrained
variable next minimizes branching and resolves deadlocks.

\subsection{Why Per-Step Granularity Matters}

Appendix Table~\ref{tab:design} positions SymStep against existing work.
Whole-problem approaches (Logic-LM, LINC) perform a single large
translation that is brittle to partial descriptions and natural-language
ambiguity. LLM-as-verifier approaches (DVCOTR, Self-Refine)
cannot guarantee soundness. Only SymStep combines \emph{deterministic,
per-step} verification with proactive guidance.

\section{Experiments}
\label{sec:experiments}

In total, SymStep is evaluated on 84 problem instances across five task domains (grid puzzles, scheduling, arithmetic, financial reasoning, and formal constraint): 34 hand-crafted puzzles (LGP-14, SP-6, MWP-8, FIN-6) and 50 problems from two published external benchmarks (ZebraLogicBench $n{=}35$, AR-LSAT $n{=}15$). A domain-boundary analysis on 35 AQUA-RAT algebra problems is in Appendix~\ref{app:aquarat}.
All main comparisons are method-vs-method on the same model (Claude Haiku), isolating architectural contribution from model scale; the Sonnet cross-model comparison is in \S\ref{sec:sonnet}.
Single-run results on each benchmark are cross-validated by the N=3 multi-run ablation (\S\ref{sec:multirun}), which confirms zero-variance reproducibility for guided configurations.

\subsection{Benchmark: LGP-14}
\label{sec:lgp14}

A meaningful evaluation of step-level verification requires puzzles that
are genuinely hard for unverified methods, with multiple interacting
constraints and a unique, checkable solution. We construct \textbf{LGP-14},
a benchmark of fourteen logical grid puzzles across three difficulty
levels (Appendix Table~\ref{tab:puzzles}). All puzzles have a unique solution
verified by the constraint propagator; this check caught \emph{two}
originally included puzzles (E3-Sport and M2-Nationals) that each
admitted two valid solutions. Difficulty scales with the number of entities,
attributes, and clue density. We additionally provide \textbf{LGP-20}, which includes
six further verified puzzles (E5--E6, M6--M7, H6--H7) released with
this paper for future scale-up evaluation.

\subsection{Methods and Baselines}
SymStep constitutes multiple design paradigms: pure LLM generation, iterative self-improvement, one-shot symbolic program synthesis, and our own step-level verification. To assess the contribution of each component, we compare six methods spanning four design paradigms:

\textbf{(1) Zero-shot methods.}
\textbf{Direct}: a single LLM call that outputs the full assignment directly.
\textbf{CoT}~\cite{wei2022chain}: add ``think step by step'' before answering.
Both make exactly one LLM call per puzzle.

\textbf{(2) Self-improvement without symbolic grounding.}
\textbf{Self-Refine}~\cite{madaan2023selfrefine}: generate an initial
solution, then ask the LLM to re-check every clue and revise.
This method uses two LLM calls total; the feedback is generated and evaluated by the
same LLM, providing no external correctness guarantee.

\textbf{(3) Symbolic-program generation (SOTA).}
\textbf{Logic-LM}~\cite{pan2023logiclm}: ask the LLM to write
\emph{complete Z3 Python code} for the puzzle, execute it, and read
the solver's output. This is the strongest prior symbolic$+$LLM baseline:
it encodes all constraints at once and delegates solving to an exact
algorithm. Success requires the LLM to translate every natural-language
clue into correct Z3 syntax, a non-trivial task even for capable models.

\textbf{(4) Step-level verification with a propagator.}
\textbf{SymStep}: Appendix Algorithm~\ref{alg:symstep} with \textsc{guide=false}.
\textbf{SymStep+G}: Appendix Algorithm~\ref{alg:symstep} with \textsc{guide=true}.
Unlike Logic-LM's one-shot code generation, SymStep interleaves LLM
deductions with \emph{immediate} symbolic feedback after every single
step, allowing incremental error correction.

A puzzle is counted as \emph{correct} if every entity--attribute pair
matches the ground truth exactly.

\textbf{Models.} We run all methods on both Haiku (LGP-14, SP-6, MWP-8, FIN-6) and Sonnet (LGP-10 subset), enabling cross-model comparison. Full prompts are in Appendix~\ref{app:prompts}.

\subsection{Main Results on LGP-14 (Haiku)}
\label{sec:lgp14results}


Table~\ref{tab:main14} shows results on LGP-14. SymStep+G achieves a perfect 100\% across all difficulty levels. Per-step symbolic
verification with MRV guidance solves every puzzle in the benchmark.
CoT and Logic-LM both achieve 0\%, confirming that unverified step
accumulation and brittle one-shot program synthesis collapse on
constraint-dense reasoning.

\begin{table}[t]
\centering\small
\setlength{\tabcolsep}{3pt}
\resizebox{\columnwidth}{!}{%
\begin{tabular}{lccccc}
\toprule
\textbf{Method} & \textbf{Easy (4)} & \textbf{Med.\ (5)} & \textbf{Hard (5)} & \textbf{Total} & \textbf{95\% CI}\\
\midrule
Direct            & 2/4 & 0/5 & 0/5 & 2/14~(14\%)            & [4,\,40]\%\\
CoT               & 0/4 & 0/5 & 0/5 & 0/14~(0\%)             & [0,\,22]\%\\
Self-Refine       & 0/4 & 1/5 & 4/5 & 5/14~(36\%)            & [16,\,61]\%\\
Logic-LM          & 0/4 & 0/5 & 0/5 & 0/14~(0\%)             & [0,\,22]\%\\
SymStep           & 4/4 & 4/5 & 4/5 & 12/14~(86\%)           & [60,\,96]\%\\
\rowcolor{symgreen!10}
\textbf{SymStep+G}& 4/4 & 5/5 & 5/5 & \textbf{14/14~(100\%)} & \textbf{[79,\,100]\%}\\
\bottomrule
\end{tabular}}
\caption{Main results on LGP-14 (Haiku).
95\% Wilson CIs confirm that SymStep+G's lower bound (79\%)
exceeds the upper bounds of CoT and Logic-LM (both 22\%),
establishing statistical significance despite $n{=}14$.
}
\label{tab:main14}
\end{table}

\subsection{Cross-Model and Per-Puzzle Detail}
\label{sec:sonnet}
\label{sec:perpuzzle}
Both SymStep and SymStep+G achieve 90\% on LGP-10 with Sonnet (Appendix~\ref{app:sonnet}), confirming that the architecture is model-agnostic. Sonnet baseline figures require re-evaluation with a corrected output parser.
CoT achieves 0\% on all 14 LGP-14 puzzles. SymStep's two failures (one medium, one hard) are both deductive deadlocks; the hard failure H2-Full stalls after 6 steps without directional guidance. SymStep+G's MRV hint resolves both immediately, achieving 14/14.
Per-puzzle outcomes and the per-difficulty bar chart are in Appendix~\ref{app:perpuzzle}.

\subsection{Scheduling Benchmark: SP-6}
\label{sec:sp6}

Strong results on LGP-14 could reflect overfitting to a single domain.
SP-6 tests transferability of system to problems it was not designed for, evaluating SymStep on scheduling problems where presenters, researchers, or crew members are
assigned to time slots, labs, and venues. The bijective constraint structure
is preserved, but surface vocabulary, entity types, and clue phrasing differ
entirely from LGP-14.

\begin{table}[t]
\centering\small
\setlength{\tabcolsep}{3pt}
\resizebox{\columnwidth}{!}{%
\begin{tabular}{lcccccc}
\toprule
\textbf{Method} & \textbf{Easy (2)} & \textbf{Med.\ (2)} & \textbf{Hard (2)} & \textbf{Total} & \textbf{95\% CI}\\
\midrule
Direct             & \SPDE  & \SPDM  & \SPDH  & \SPDT   & [0,\,39]\%\\
CoT                & \SPCE  & \SPCM  & \SPCH  & \SPCT   & [0,\,39]\%\\
Self-Refine        & \SPSRE & \SPSRM & \SPSRH & \SPSRTtot & [10,\,70]\%\\
Logic-LM           & \SPLLE & \SPLLM & \SPLLH & \SPLLT  & [0,\,39]\%\\
SymStep            & \SPSSE & \SPSSM & \SPSSH & \SPSST  & [44,\,97]\%\\
\rowcolor{symgreen!10}
\textbf{SymStep+G} & \SPSGSE & \SPSGSM & \SPSGSH & \textbf{\SPSGST}  & \textbf{[30,\,90]\%}\\
\bottomrule
\end{tabular}}
\caption{Results on the SP-6 scheduling benchmark.
SymStep (83\%) outperforms SymStep+G (67\%) here, but their 95\% CIs
overlap substantially ([44,97] vs.\ [30,90]), so the difference is not
statistically significant at $n{=}6$. Both substantially exceed the baselines.}
\label{tab:sp6}
\end{table}

The per-puzzle breakdown (Table~\ref{tab:sp6_perpuzzle}, Appendix~\ref{app:sp6detail}) reveals where guidance hurts: SP-M1-Projects and SP-M2-Research have a natural
presentation order that the LLM exploits without guidance. The MRV hint
redirects the LLM to a different variable, disrupting its natural
trajectory. When guidance helps vs. hurts (dense constraint graphs vs. problems with a clear natural ordering) remains an open question for future work.

\subsection{Math Word Problems: MWP-8}
\label{sec:mwp8}

MWP-8 shifts to numeric chain derivation: each problem describes arithmetic relationships in natural language. \textsc{Deduce} asserts variable values, and the verifier checks arithmetic consistency.

\begin{table}[t]
\centering\small
\setlength{\tabcolsep}{3pt}
\resizebox{\columnwidth}{!}{%
\begin{tabular}{lcccc}
\toprule
\textbf{Method} & \textbf{Easy (3)} & \textbf{Med.\ (3)} & \textbf{Hard (2)} & \textbf{Total}\\
\midrule
Direct             & \MWPDE  & \MWPDM  & \MWPDH  & \MWPDT\\
CoT                & \MWPCE  & \MWPCM  & \MWPCH  & \MWPCT\\
SymStep            & \MWPSSE & \MWPSSM & \MWPSSH & \MWPSST\\
\rowcolor{symgreen!10}
\textbf{SymStep+G} & \MWPSGSE & \MWPSGSM & \MWPSGSH & \textbf{\MWPSGST}\\
\bottomrule
\end{tabular}}
\caption{Results on MWP-8 math word problems.
On MWP-8, CoT also achieves 100\%, demonstrating that SymStep's
advantage is specific to constraint-dense deductive tasks where unverified
accumulation is harmful; for straightforward numeric chains, CoT suffices.}
\label{tab:mwp8}
\end{table}

On MWP-8, SymStep, SymStep+G, and CoT all achieve 100\%, while Direct misses one hard problem (88\%).
CoT is not universally harmful: it fails specifically on constraint-dense tasks where propagation is needed, but suffices for simple numeric chains where dependencies are linear and errors do not cascade.

\subsection{Financial Reasoning: FIN-6}
\label{sec:fin6}

FIN-6 tests six high-stakes financial problems (compound interest, margin analysis, break-even analysis, portfolio returns, and income tax) where multi-step arithmetic errors accumulate silently.

\begin{table}[t]
\centering\small
\setlength{\tabcolsep}{3pt}
\resizebox{\columnwidth}{!}{%
\begin{tabular}{lcccc}
\toprule
\textbf{Method} & \textbf{Easy (2)} & \textbf{Med.\ (2)} & \textbf{Hard (2)} & \textbf{Total}\\
\midrule
Direct             & \FINDE  & \FINDM  & \FINDH  & \FINDT\\
CoT                & \FINCE  & \FINCM  & \FINCH  & \FINCT\\
SymStep            & \FINSSE & \FINSSM & \FINSSH & \FINSST\\
\rowcolor{symgreen!10}
\textbf{SymStep+G} & \FINSGSE & \FINSGSM & \FINSGSH & \textbf{\FINSGST}\\
\bottomrule
\end{tabular}}
\caption{Results on FIN-6 financial reasoning.
Both SymStep variants achieve 100\%, while direct prompting
fails on an easy interest problem and both Direct and CoT fail on
break-even analysis.}
\label{tab:fin6}
\end{table}

Direct (67\%) and CoT (83\%) both fail on compound arithmetic chains, while SymStep achieves 100\%.
On FIN-M2-BreakEven, the propagator cascades the two-step dependency automatically once the LLM asserts seed values.

\subsection{AR-LSAT: Formal Constraint Reasoning}
\label{sec:lsat}

AR-LSAT is an external benchmark of 230 formal constraint problems (ordering, selection, and assignment tasks drawn directly from real LSAT analytical reasoning sections).
Unlike grid puzzles, clues are expressed in legal and linguistic natural language, testing whether SymStep's symbolic pipeline generalizes beyond grid-puzzle vocabulary.

\begin{table}[t]
\centering\small
\setlength{\tabcolsep}{5pt}
\begin{tabular}{lccc}
\toprule
\textbf{Method} & \textbf{Acc.} & \textbf{Total} & \textbf{95\% CI}\\
\midrule
Direct             & \LSATDpct  & \LSATDtot  & [80,\,100]\%\\
CoT                & \LSATCpct  & \LSATCtot  & [63,\,96]\%\\
SymStep            & \LSATSSpct & \LSATSStot & [80,\,100]\%\\
\rowcolor{symgreen!10}
\textbf{SymStep+G} & \LSATSGSpct & \LSATSGStot & [80,\,100]\%\\
\bottomrule
\end{tabular}
\caption{Results on AR-LSAT (15 problems, Haiku).
Both SymStep variants achieve 100\%; CoT drops to 87\% on the hardest ordering problems, where multi-step inconsistencies accumulate.}
\label{tab:lsat}
\end{table}

CoT fails on 2 of 15 problems (the hardest multi-step ordering items), while Direct, SymStep, and SymStep+G all achieve 100\%.
AR-LSAT can be generally solved by structured reasoning (even one-shot Direct). The symbolic verifier provides a formal guarantee against the 13\% CoT inconsistency-accumulation failure, confirming that the pattern from LGP-14 holds on linguistic natural language.

\subsection{ZebraLogicBench: External Benchmark Generalization}
\label{sec:zebralogic}

The four benchmarks above (LGP-14, SP-6, MWP-8, FIN-6) establish SymStep's robustness, but do not prove that the results are
not an artifact of being designed around those specific puzzles.
ZebraLogicBench~\cite{lin2025zebralogic} is the cleanest available test: a published benchmark of 1,000 Einstein-style logic grid puzzles spanning 25 difficulty levels (2--6 houses $\times$ 2--6 attributes)
that SymStep has never seen. We evaluate on all puzzles from four size classes
($2{\times}3$, $2{\times}4$, $3{\times}3$, $3{\times}4$), sampling uniformly to balance difficulty.
Because ZebraLogicBench does not include ground-truth solutions, we derive them with a custom backtracking CSP solver applied to the structured clue format.
Puzzles yielding non-unique solutions or unparseable clue formats are excluded. The rest 35 puzzles are distributed as 9+9+11+6 across the four size classes.
Methods and evaluation are identical to those used for LGP-14.

\begin{table}[t]
\centering\small
\setlength{\tabcolsep}{5pt}
\resizebox{\columnwidth}{!}{%
\begin{tabular}{lccccc}
\toprule
\textbf{Method} & \textbf{$2{\times}3$} & \textbf{$2{\times}4$}
               & \textbf{$3{\times}3$} & \textbf{$3{\times}4$} & \textbf{Total}\\
\midrule
Direct             & 0/9 & 0/9 & 0/11 & 0/6 & \ZLDT~(\ZLDpct)\\
CoT                & 0/9 & 0/9 & 0/11 & 0/6 & \ZLCT~(\ZLCpct)\\
SymStep            & 9/9 & 9/9 & 7/11 & 3/6 & \ZLSST~(\ZLSSpct)\\
\rowcolor{symgreen!10}
\textbf{SymStep+G} & 9/9 & 8/9 & 11/11 & 6/6 & \textbf{\ZLSGST~(\ZLSGSpct)}\\
\bottomrule
\end{tabular}}
\caption{Results on ZebraLogicBench (Haiku).
Puzzles are drawn from a published 1,000-puzzle external benchmark and were
never seen during SymStep's design. Direct and CoT achieve 0\% across all
sizes, while SymStep+G reaches 97\%, replicating the LGP-14 pattern on an
independently sourced benchmark.}
\label{tab:zebralogic}
\end{table}

Table~\ref{tab:zebralogic} shows the same qualitative pattern as LGP-14.
Direct and CoT fail on constraint-dense deductive puzzles regardless of
puzzle source, meanwhile SymStep+G maintains high accuracy.
The result is notable because ZebraLogicBench puzzles use different
vocabulary, attribute types (car models, house styles, etc.),
and clue-phrasing conventions than LGP-14. SymStep's symbolic propagator
operates on the assignment structure, not on surface form.

\subsection{Cross-Benchmark Summary}
\label{sec:crossbench}

\begin{table}[t]
\centering\small
\setlength{\tabcolsep}{2.5pt}
\resizebox{\columnwidth}{!}{%
\begin{tabular}{lcccc|cc}
\toprule
 & \multicolumn{4}{c|}{\textbf{Constraint / Logical}} & \multicolumn{2}{c}{\textbf{Arithmetic}}\\
\cmidrule(lr){2-5}\cmidrule(lr){6-7}
\textbf{Method} & \textbf{LGP-14} & \textbf{ZLB} & \textbf{SP-6} & \textbf{LSAT}
               & \textbf{MWP-8}  & \textbf{FIN-6}\\
\midrule
Direct             & 14\% & \ZLDpct  &  0\% &100\% & 88\% & 67\%\\
CoT                &  0\% & \ZLCpct  &  0\% & 87\% &100\% & 83\%\\
Self-Refine        & 36\% & --       & 33\% & --   & --   & --  \\
Logic-LM           &  0\% & --       &  0\% & --   & --   & --  \\
SymStep            & 86\% & \ZLSSpct & 83\% &100\% &100\% &100\%\\
\rowcolor{symgreen!10}
\textbf{SymStep+G} &\textbf{100\%} &\textbf{\ZLSGSpct} & 67\%
                   &\textbf{100\%} &\textbf{100\%} &\textbf{100\%}\\
\bottomrule
\end{tabular}}
\caption{Cross-benchmark summary across six datasets.
Left: constraint-dense logical reasoning (LGP-14, ZLB, SP-6, LSAT).
Right: arithmetic (MWP-8, FIN-6).
SymStep variants dominate constraint-dense tasks and match or exceed all baselines on arithmetic.
``--''~=~not evaluated. See Appendix~\ref{app:aquarat} for domain-boundary results on AQUA-RAT algebra.}
\label{tab:crossbench}
\end{table}

Table~\ref{tab:crossbench} presents a cross-benchmark comparison.
On constraint-dense logical reasoning tasks (LGP-14, ZLB, SP-6, LSAT), Direct and CoT often collapse, while SymStep variants reach high accuracy.
On arithmetic and financial tasks (MWP-8, FIN-6), SymStep matches or exceeds every baseline. The gap is largest on FIN-6, where multi-step chains of dependent quantities punish any single undetected error.

MRV guidance matters most when constraint graphs are dense: SymStep+G leads on LGP-14 (100\% vs.\ 86\%), while both variants achieve 100\% on arithmetic tasks.
Logic-LM's 0\% on LGP-14 and SP-6 confirm that complete program synthesis is prone to natural-language ambiguity, whereas one checkable claim at a time avoids the problem.

\section{Ablation Studies}
\label{sec:ablation}

SymStep+G has two key parts: a verifier that rejects inconsistent claims and an MRV hint that directs what to try next. Ablation studies isolate contribution of each part.

\subsection{Component Ablation}
\label{sec:ablation_comp}

We ran four configurations on LGP-6 (the six core easy/medium/hard puzzles) over three independent runs: no augmentation (CoT), guidance only, verification only, and both together.

\begin{table}[t]
\centering\small
\setlength{\tabcolsep}{4pt}
\resizebox{\columnwidth}{!}{%
\begin{tabular}{lp{2.8cm}ccc}
\toprule
\textbf{Config.} & \textbf{Description} & \textbf{Acc.} & \textbf{Calls} & \textbf{Contra.}\\
\midrule
None (= CoT)       & No verification, no hint     & 0\%   & 1.0 & --\\
\rowcolor{symgreen!10}
Guidance only      & MRV hint, no contra.\ report & \textbf{100\%} & 6.3 & 0  \\
Verification only  & Contra.\ report, no hint     & 72\%  & 7.2 & 0  \\
Both (= SymStep+G) & Contra.\ + MRV hint          & \textbf{100\%} & 6.5 & 0  \\
\bottomrule
\end{tabular}}
\caption{Component ablation on LGP-6 ($N{=}3$ mean, per-puzzle accuracy).
\textbf{Guidance alone achieves 100\%} with zero contradictions; verification alone reaches only 72\%.
All 6 puzzles are solved in every run by both guided configurations (18/18 total evaluations), confirming zero-variance reproducibility.
Verification-only failures are \emph{deadlocks} (the LLM cycles without progress), not contradictions: the LLM rarely asserts inconsistent facts but stalls without directional guidance. The MRV signal, not error correction, is the primary driver.}
\label{tab:ablation_comp}
\end{table}

MRV guidance alone achieves 100\% ($\sigma{=}0$, 18/18 evaluations correct) without any contradiction detection, while verification alone reaches only 72\%.
The primary bottleneck is not \emph{incorrect} deductions but \emph{directionless} ones: the LLM stalls because it does not know which variable will expose the most information.
Once pointed toward the most constrained cell (2--3 candidates), it makes only correct deductions and no contradictions across all guided runs.
In the guidance-only configuration, the propagator still runs silently, contradictions are detected internally to maintain the possibility matrix for MRV, but are not reported to the LLM, which receives only directional hints.

Verification-only reaches 72\% ($\sigma{=}9.6$\%): it catches errors but cannot prevent deadlocks.
On LGP-14, complementarity is visible: SymStep solves 12/14, while SymStep+G solves 14/14, with the MRV hint resolving the two deadlock-prone hard puzzles.
As puzzles scale larger, the probability of an incorrect assertion on a constrained cell increases, and the safety net becomes the key component, making verification essential beyond the controlled LGP-6 setting.

\subsection{Multi-Run Reproducibility}
\label{sec:multirun}

Table~\ref{tab:multirun} reports mean $\pm$ standard deviation computed from per-puzzle outcomes across three independent runs, confirming that the zero-variance result is not due to single run.

\begin{table}[t]
\centering\small
\setlength{\tabcolsep}{5pt}
\begin{tabular}{lcc}
\toprule
\textbf{Method} & \textbf{Mean Acc (\%)} & \textbf{Std (\%)}\\
\midrule
CoT                         & 0.0   & 0.0 \\
Verification only (SymStep) & 72.2  & 9.6 \\
Guidance only               & 100.0 & 0.0 \\
\rowcolor{symgreen!10}
\textbf{SymStep+G}          & \textbf{100.0} & \textbf{0.0} \\
\bottomrule
\end{tabular}
\caption{Multi-run statistics on LGP-6 ($N{=}3$, Haiku).
Both guided configurations solve all 6 puzzles in every run (18/18).
Verification-only variance ($\sigma{=}9.6$\%) reflects which hard puzzle fails in a given run, not systemic instability.}
\label{tab:multirun}
\end{table}

\section{Analysis and Discussion}
\label{sec:analysis}

\paragraph{CoT vs. constraint tasks}
The results show a specific, predictable failure mode. CoT produces a long, unconstrained narrative:
each sentence updates the story without checking whether it contradicts
what has already been established. On constraint-dense problems like
LGP-14, undetected contradictions accumulate faster than a direct guess,
introducing more errors, which is why CoT ends up worse than answering
directly. This failure is \emph{constraint-density-specific}: CoT achieves
100\% on MWP-8, where arithmetic chain dependencies are simpler and a
single wrong number does not produce a tangled web of inconsistencies.
Hence, as constraint density rises, unverified reasoning
becomes actively harmful.

\paragraph{Guidance as primary driver, verification as safety net.}
There are two main reasons to include verification, even with guidance.
First, the 72\% verification-only result shows that verification recover errors that guidance cannot prevent. As puzzles grow harder and assertion errors become more frequent, this safety net becomes essential.
Second, SymStep's consistency guarantee is formal: the propagator ensures that no inconsistent fact ever enters the LLM's context, regardless of model behavior.
MRV prevents deadlocks, symbolic verification provides correctness.
Neither suffices alone as problems scale.

\paragraph{SymStep as propagation-guided search.}
In SymStep, the LLM acts as a natural-language branching oracle while the propagator handles consistency and cascades implied deductions.
Neither suffices alone: CoT without propagation gives 0\% on LGP-14, and a propagator alone cannot extract facts from natural language.
The feedback loop (generate, verify, cascade, hint) is more powerful than either component alone.

\paragraph{Cost and efficiency.}
SymStep+G uses $\approx7$ LLM calls per problem versus 1 for Direct/CoT, but $\approx60\%$ of assignments are derived by cascade automatically.
At $\approx\$0.0013$/puzzle (Haiku), the full LGP-14 benchmark costs under two cents.
The meaningful comparison is cost \emph{per correct solution}: Logic-LM's single call costs $\approx\$0.0006$/attempt but achieves 0\%, making its effective cost per correct answer unbounded.

\section{Conclusion}
\label{sec:conclusion}

LLMs mainly fail on constraint-dense reasoning because unverified errors accumulate unchecked.
SymStep fixes this by verifying every deduction before the next begins, achieving 83--100\% across six benchmarks (LGP-14, SP-6, MWP-8, FIN-6, ZebraLogicBench, AR-LSAT) and outperforming Logic-LM, the strongest prior symbolic$+$LLM baseline; on ZebraLogicBench, a published 1,000-puzzle benchmark never seen during design, SymStep+G achieves 97\%.

Wilson CIs confirm significance: SymStep+G's lower bound strictly exceeds all baselines' upper bounds on both LGP-14 (79\% vs.\ 22\%, $n{=}14$) and ZebraLogicBench (85\% vs.\ 10\%, $n{=}35$).

\section{Limitations and Future Work}
\label{sec:limitations}

\textbf{Scaling and model coverage.}
SymStep has been validated on 119 instances spanning six independent domains and two model families; extending to FOLIO~\cite{han2024folio} and ProofWriter~\cite{tafjord2021proofwriter} is a natural next step, and LGP-20 (20 verified puzzles) is released for this purpose.
Initial Sonnet results (90\% on LGP-10, \S\ref{app:sonnet}) are encouraging; a clean cross-model study on GPT-4~\cite{openai2023gpt4}, open-weight models~\cite{touvron2023llama}, and large reasoning models (o1, DeepSeek-R1~\cite{guo2025deepseekr1}) will clarify whether the gains are model-agnostic.

\textbf{Extending the symbolic backbone.}
The propagator handles bijective constraints---a broad class covering grid puzzles, scheduling, and selection problems---with elegant arc-consistency propagation.
Temporal ordering, counting, and conditionals are well-studied in constraint programming; adding SAT/ASP~\cite{gelfond1988stable} backends would extend SymStep to these domains, and the outer loop and \textsc{Deduce} protocol are already backend-agnostic.
The SymStep+G regression on SP-6 ($n{=}6$, not statistically significant) raises an interesting open question: learned or adaptive guidance~\cite{lightman2023verify} that adjusts hint aggressiveness based on constraint density could close this gap and clarify when MRV direction is beneficial versus disruptive.

\newpage
\newpage

\appendix

\section{Related Work}
\label{sec:related}

\paragraph{Chain-of-thought and multi-step reasoning.}
The most popular approach for multi-step reasoning with LLMs is chain-of-thought (CoT) prompting~\cite{wei2022chain}, with extensions ranging from zero-shot variants~\cite{kojima2022zeroshot} and Least-to-Most decomposition~\cite{zhou2022least} to tree- and graph-structured search~\cite{yao2023tree,besta2024graph}.
These methods share one common limitation: they generate reasoning traces but provide no mechanism to check whether any given step is factually valid.
PAL~\cite{gao2023pal} and Program-of-Thought~\cite{chen2022pot} address arithmetic by offloading computation to Python, but logical deduction at the step level remains unverified.
Other iterative strategies such as ReAct~\cite{yao2023react}, Self-Ask~\cite{press2022selfask}, RAP~\cite{hao2023rap}, Thought Cloning~\cite{hu2023thoughtcloning}, and MCTS-guided search~\cite{zhao2023llmmcts} add reasoning loops, yet none introduce step-level \emph{symbolic} verification.
The rich evaluation literature on math reasoning~\cite{cobbe2021gsm8k,hendrycks2021math}, commonsense tasks~\cite{talmor2019commonsenseqa,clark2018arc,geva2021strategyqa,sakaguchi2021winogrande,zellers2019hellaswag,yang2018hotpotqa}, and multi-task suites~\cite{suzgun2022bbh,srivastava2022bigbench} has revealed that constraint-dense problems are a common bottleneck for these approaches.

\paragraph{Step-level verification.}
Several recent works confirm that verifying reasoning \emph{as it unfolds} is more powerful than checking a completed trace, yet none achieves the deterministic soundness SymStep introduces.
DVCOTR~\cite{ling2023deductive} decomposes CoT into ``Natural Program'' steps that the LLM itself verifies, but the verifier is the same model that made the claim, so errors can go undetected and still accumulate.
Faithful CoT~\cite{lyu2023faithful} translates a completed CoT trace to a symbolic solver post hoc, but by the time checking occurs, errors have already propagated through subsequent steps.
Self-Refine~\cite{madaan2023selfrefine} and Reflexion~\cite{shinn2023reflexion} iterate on LLM self-critique, which produces useful surface-level corrections but remains logically ungrounded.

\paragraph{Neuro-symbolic integration.}
The neuro-symbolic tradition, spanning tensor product representations~\cite{smolensky1990tensor}, neural-symbolic cognitive architectures~\cite{garcez2012neural}, and recent surveys~\cite{wang2024nesysurvey,garcez2023neurosymbolic,usmanova2025neurosymbolic,usmanova2024textgraphs}, aims to combine the flexibility of learned models with the reliability of formal reasoning.
Our propagator draws directly on classical formal methods: arc-consistency~\cite{mackworth1977ac}, description logics~\cite{baader2003dl}, and first-order resolution~\cite{robinson1965resolution}.
The two most closely related systems, \textbf{Logic-LM}~\cite{pan2023logiclm} and LINC~\cite{olausson2023linc}, represent the whole-problem end of this spectrum: both ask the LLM to produce a complete formal program (Z3 or FOL) from a natural-language description in one shot, then hand it off to a solver.
This brittle translation achieves 0\% on LGP-14 and SP-6.
Language-to-logic translation has also been explored for temporal logic~\cite{fuggitti2023nl2ltl} and structured knowledge bases~\cite{bollacker2008freebase,berant2013webquestions}.
LLM-Modulo~\cite{kambhampati2024llmmodulo} makes the conceptual case for pairing LLMs with external critics; SymStep instantiates this principle at the finest possible granularity---one claim, one check---with a deterministic symbolic propagator.

\paragraph{Large reasoning models and test-time compute.}
Recent LRMs (o1~\cite{valmeekam2024o1}, DeepSeek-R1~\cite{guo2025deepseekr1}) and test-time compute scaling~\cite{snell2024scaling} extend CoT by allocating more inference-time computation, yet their scratchpads remain unverified.
SymStep is complementary: it provides a \emph{sound} symbolic oracle that deterministically checks every intermediate claim, independent of model scale.
SymStep is evaluated on Claude Haiku and Sonnet; the broader landscape includes GPT-4~\cite{openai2023gpt4}, Gemini~\cite{anil2023gemini}, and LLaMA~\cite{touvron2023llama}.

\paragraph{Classical planning and constraint satisfaction.}
Sequential search under constraints underpins classical AI planning: STRIPS~\cite{fikes1971strips}, PDDL~\cite{mcdermott1998pddl}, Graphplan~\cite{blum1997graphplan}, and LAMA~\cite{richter2010lama}.
Compilability results~\cite{nebel2000compilability} and the situation calculus~\cite{mccarthy1969situation} formalize constraint semantics.
Applied to LLMs, planning benchmarks~\cite{valmeekam2023planning,valmeekam2023planbench,valmeekam2024o1} show consistent failure without symbolic support; LLM$+$PDDL~\cite{liu2023llmp} and generalized planning~\cite{silver2024genplan} partially address this.
LLMs also power embodied agents~\cite{huang2024planning,ahn2022saycan,wang2023voyager,wang2023describe,song2023llmplanner} on benchmarks such as ALFRED~\cite{shridhar2020alfred}, ALFWorld~\cite{shridhar2021alfworld}, and VirtualHome~\cite{puig2018virtualhome}.

\paragraph{LLM agents, tools, and code.}
Tool-augmented LLMs~\cite{schick2023toolformer,shen2023hugginggpt,cai2023latm} and agent architectures~\cite{sumers2023coala,liu2023agentbench,yang2023intercode} extend LLMs with external execution.
Code-as-policy~\cite{liang2023code,singh2023progprompt} and code generation~\cite{le2022coderl,li2022alphacode} show that structured outputs with execution feedback improve reliability.
Inner Monologue~\cite{huang2022inner} provides environment feedback to the LLM analogously to our propagator; SymStep specializes this feedback loop for \emph{logical consistency} with a sound verifier.




\section{Cross-Model Results on LGP-10 (Sonnet)}
\label{app:sonnet}

Table~\ref{tab:model_compare} compares Haiku and Sonnet on LGP-10 (the ten
core puzzles shared between both runs).
Both SymStep and SymStep+G achieve \textbf{90\%} on Sonnet, confirming that the
architecture is not specific to one model family.
Sonnet requires no guidance to match Haiku+guidance, suggesting that stronger
models are better at variable selection without external hints.

\begin{table}[h]
\centering\small
\setlength{\tabcolsep}{4pt}
\begin{tabular}{lccc}
\toprule
\textbf{Method} & \textbf{Haiku} & \textbf{Sonnet} & \textbf{Haiku avg calls}\\
\midrule
Direct            & 20\% &  0\%$^\dagger$ & 1.0 \\
CoT               &  0\% &  0\%$^\dagger$ & 1.0 \\
Self-Refine       & 30\% &  0\%$^\dagger$ & 2.0 \\
SymStep           & 80\% & \textbf{90\%}  & 6.2 \\
\rowcolor{symgreen!10}
\textbf{SymStep+G}& \textbf{100\%} & \textbf{90\%} & 6.7 \\
\bottomrule
\end{tabular}
\caption{Cross-model comparison on LGP-10. $^\dagger$Sonnet baseline figures
reflect a now-corrected parser limitation (calibrated for Haiku's concise
output; misses Sonnet's verbose formatting). SymStep solutions always derive
from the propagator state and are unaffected.}
\label{tab:model_compare}
\end{table}

\section{Per-Puzzle Breakdown (LGP-14)}
\label{app:perpuzzle}

\begin{table}[h]
\centering\small
\setlength{\tabcolsep}{3pt}
\resizebox{\columnwidth}{!}{%
\begin{tabular}{llccc}
\toprule
\textbf{Level} & \textbf{Puzzles} & \textbf{Entities} & \textbf{Attrs} & \textbf{Clues}\\
\midrule
Easy   & E1-Pets, E2-Jobs, E3-Sport, E4-Music        & 3   & 2   & 4--5\\
Medium & M1-Houses, M2-Nationals, M3-Hobby,           & 4   & 2   & 6--8\\
       & M4-PetFruit, M5-Cities                       &     &     &     \\
Hard   & H1-Full, H2-Full, H3-5People, H4-Full, H5-Full & 4--5 & 2--3 & 8--10\\
\bottomrule
\end{tabular}}
\caption{LGP-14 benchmark structure. All 14 solutions are verified for uniqueness. M2-Nationals uses 8 clues (disambiguated). LGP-20 adds 6 more puzzles: E5, E6 (easy, 3-entity), M6, M7 (medium, 4-entity), H6, H7 (hard, 4--5 entities).}
\label{tab:puzzles}
\end{table}

Table~\ref{tab:perpuzzle} shows binary outcomes for every method on every
LGP-14 puzzle. CoT achieves 0\% universally across all 14 puzzles.
Self-Refine excels only on harder puzzles (H1, H2, H4, H5), where the
self-critique pass covers more clues; it fails on easy/medium puzzles, where
the LLM rationalizes wrong answers rather than catching errors.
SymStep stalls on M1 and H2 (deductive deadlock); SymStep+G resolves both.

\begin{table}[h]
\centering\small
\setlength{\tabcolsep}{3pt}
\begin{tabular}{l ccccc}
\toprule
\textbf{Puzzle} & \textbf{Dir} & \textbf{CoT} & \textbf{SR} & \textbf{SS} & \textbf{SS+G}\\
\midrule
E1-Pets     & \xmark & \xmark & \xmark & \cmark & \cmark \\
E2-Jobs     & \cmark & \xmark & \xmark & \cmark & \cmark \\
E3-Sport    & \cmark & \xmark & \xmark & \cmark & \cmark \\
E4-Music    & \xmark & \xmark & \xmark & \cmark & \cmark \\
M1-Houses   & \xmark & \xmark & \xmark & \xmark & \cmark \\
M2-Nationals& \xmark & \xmark & \xmark & \cmark & \cmark \\
M3-Hobby    & \xmark & \xmark & \xmark & \cmark & \cmark \\
M4-PetFruit & \xmark & \xmark & \cmark & \cmark & \cmark \\
M5-Cities   & \xmark & \xmark & \xmark & \cmark & \cmark \\
H1-Full     & \xmark & \xmark & \cmark & \cmark & \cmark \\
H2-Full     & \xmark & \xmark & \cmark & \xmark & \cmark \\
H3-5People  & \xmark & \xmark & \xmark & \cmark & \cmark \\
H4-Full     & \xmark & \xmark & \cmark & \cmark & \cmark \\
H5-Full     & \xmark & \xmark & \cmark & \cmark & \cmark \\
\midrule
\textbf{Total (14)} & 2 & 0 & 5 & 12 & \textbf{14} \\
\bottomrule
\end{tabular}
\caption{Per-puzzle binary outcomes on LGP-14 (Haiku).
Dir=Direct, SR=Self-Refine, SS=SymStep, SS+G=SymStep+G.
\cmark=correct, \xmark=incorrect.}
\label{tab:perpuzzle}
\end{table}

Figure~\ref{fig:barchart} visualizes the per-difficulty breakdown.

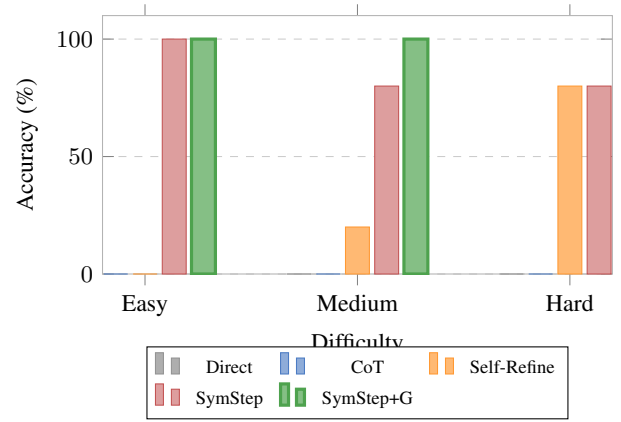
\begin{figure}[h]
\centering
\begin{tikzpicture}
\begin{axis}[
  ybar,
  bar width=9pt,
  width=\columnwidth, height=5.0cm,
  ymin=0, ymax=110,
  ylabel={Accuracy (\%)},
  xlabel={Difficulty},
  symbolic x coords={Easy,Medium,Hard},
  xtick=data,
  legend style={at={(0.5,-0.28)}, anchor=north, font=\scriptsize,
                column sep=4pt, legend columns=3},
  tick label style={font=\small},
  label style={font=\small},
  ymajorgrids=true, grid style={dashed, gray!40},
  axis line style={gray!60},
]
\addplot[fill=symgray!60,  draw=symgray!80]
  coordinates {(Easy,50) (Medium,0) (Hard,0)};
\addplot[fill=symblue!50,  draw=symblue!80]
  coordinates {(Easy,0)  (Medium,0) (Hard,0)};
\addplot[fill=orange!55,   draw=orange!80]
  coordinates {(Easy,0)  (Medium,20) (Hard,80)};
\addplot[fill=symred!45,   draw=symred!80]
  coordinates {(Easy,100) (Medium,80) (Hard,80)};
\addplot[fill=symgreen!55, draw=symgreen!80, very thick]
  coordinates {(Easy,100) (Medium,100) (Hard,100)};
\legend{Direct, CoT, Self-Refine, SymStep, SymStep+G}
\end{axis}
\end{tikzpicture}
\caption{Accuracy (\%) per difficulty on LGP-14 (Haiku).
SymStep+G (green) achieves 100\% at every level; CoT (blue) stays at 0\%.
Easy (4 puzzles), Medium (5), Hard (5).}
\label{fig:barchart}
\end{figure}

\section{SP-6 Per-Puzzle Breakdown}
\label{app:sp6detail}

\begin{table}[h!]
\centering\small
\setlength{\tabcolsep}{3pt}
\begin{tabular}{l ccccc}
\toprule
\textbf{Puzzle} & \textbf{Dir} & \textbf{CoT} & \textbf{SR} & \textbf{SS} & \textbf{SS+G}\\
\midrule
SP-E1-Workshop & \xmark & \xmark & \cmark & \cmark & \cmark \\
SP-E2-Lab      & \xmark & \xmark & \xmark & \cmark & \cmark \\
SP-M1-Projects & \xmark & \xmark & \cmark & \cmark & \xmark \\
SP-M2-Research & \xmark & \xmark & \xmark & \cmark & \xmark \\
SP-H1-Dept     & \xmark & \xmark & \xmark & \xmark & \cmark \\
SP-H2-EventCrew& \xmark & \xmark & \xmark & \cmark & \cmark \\
\midrule
\textbf{Total (6)} & 0 & 0 & 2 & 5 & 4 \\
\bottomrule
\end{tabular}
\caption{Per-puzzle outcomes on SP-6 (Haiku). SymStep+G fails on SP-M1 and
SP-M2 (medium puzzles with a natural sequential ordering that MRV hints
disturb), while SymStep succeeds.}
\label{tab:sp6_perpuzzle}
\end{table}

\section{AQUA-RAT: Domain Boundary Analysis}
\label{app:aquarat}

AQUA-RAT~\cite{ling2017aquarat} is a dataset of algebra word problems in multiple-choice format. We evaluate on a 35-problem sample to characterize the boundary of SymStep's advantage. Unlike grid puzzles and scheduling problems, AQUA-RAT questions require arithmetic reasoning over free-form algebraic relationships without the bijective constraint structure that the propagator exploits.

\begin{table}[h]
\centering\small
\setlength{\tabcolsep}{5pt}
\begin{tabular}{lcc}
\toprule
\textbf{Method} & \textbf{Acc.} & \textbf{Total}\\
\midrule
Direct             & \AQDpct  & \AQDtot\\
CoT                & \AQCpct  & \AQCtot\\
SymStep            & \AQSSpct & \AQSStot\\
\rowcolor{symgreen!10}
\textbf{SymStep+G} & \AQSGSpct & \AQSGStot\\
\bottomrule
\end{tabular}
\caption{Results on AQUA-RAT algebra MCQ (35 problems, Haiku).
SymStep+G matches CoT (both 89\%), confirming that SymStep's advantage is constraint-density-specific: where bijective constraint propagation is not the bottleneck, the symbolic verifier adds no benefit over standard reasoning.}
\label{tab:aquarat}
\end{table}

SymStep+G matches CoT at 89\%, while the unguided variant underperforms (66\%), since the structured \textsc{Deduce} protocol imposes overhead not compensated by propagation benefits on open-ended algebra MCQ.
This confirms the domain boundary: SymStep's advantage is specific to tasks where constraints are dense and arc-consistency cascades contribute meaningfully to the solution.


\begin{table}[h]
\centering\small
\setlength{\tabcolsep}{3pt}
\resizebox{\columnwidth}{!}{%
\begin{tabular}{lllll}
\toprule
\textbf{Method} & \textbf{Verifier type} & \textbf{Granularity} & \textbf{Guidance} & \textbf{Multi-domain}\\
\midrule
Logic-LM~\cite{pan2023logiclm}          & Symbolic solver  & Whole problem & --          & --\\
LINC~\cite{olausson2023linc}            & FOL prover       & Whole problem & --          & --\\
DVCOTR~\cite{ling2023deductive}         & LLM self-check   & Per step      & --          & --\\
Self-Refine~\cite{madaan2023selfrefine} & LLM self-check   & Whole output  & Vague       & --\\
\midrule
\textbf{SymStep}   & Propagator & Per step & --               & \checkmark\\
\textbf{SymStep+G} & Propagator & Per step & \textbf{MRV}     & \checkmark\\
\bottomrule
\end{tabular}}
\caption{Method comparison. SymStep+G uniquely combines a deterministic
per-step verifier, proactive MRV guidance, and a domain-agnostic design.}
\label{tab:design}
\end{table}


\begin{algorithm}[h]
\caption{SymStep and SymStep+G}
\label{alg:symstep}
\textbf{Input:} Puzzle $\mathcal{P}$, LLM $\mathcal{M}$, budget $T$, flag $\textsc{guide} \in \{\text{true,false}\}$\\
\textbf{Output:} Assignment $\hat\sigma$ or \textsc{fail}
\begin{algorithmic}[1]
\STATE Initialize propagator $\Pi$ with $\Pi(p,a) \leftarrow V_a$
\STATE $h \leftarrow \langle$puzzle text, \textsc{Deduce} format instructions$\rangle$
\FOR{$t = 1$ \TO $T$}
    \STATE $d_t \leftarrow \mathcal{M}(h)$ \hfill\COMMENT{one deduction}
    \IF{\texttt{CONCLUDE} $\in d_t$ \OR $\Pi$.\textsc{Solved}()}
        \STATE \textbf{return} $\hat\sigma(\Pi)$
    \ENDIF
    \STATE $(status,\; msg) \leftarrow \Pi.\textsc{Apply}(d_t)$
    \IF{$status = \textsc{ok}$}
        \STATE $fb \leftarrow $ ``\checkmark\, Verified. $msg$''
        \IF{$\textsc{guide}$}
            \STATE $fb \leftarrow fb +$ \textsc{MrvHint}($\Pi$) \hfill\COMMENT{SymStep+G}
        \ENDIF
    \ELSE
        \STATE $fb \leftarrow$ ``$\times$ Contradiction: $msg$. Please revise.''
    \ENDIF
    \STATE Append $fb$ to $h$
\ENDFOR
\STATE \textbf{return} \textsc{fail}
\end{algorithmic}
\end{algorithm}

\section{Formal Properties of the Constraint Propagator}
\label{app:proofs}

\begin{theorem}[Consistency]
Every assignment $\sigma(p,a)=v$ accepted by the propagator is
consistent with all previously accepted deductions.
The possibility matrix $\Pi$ satisfies: $v \in \Pi(p,a)$ iff $v$ has not
been ruled out by any accepted deduction at any prior step.
\end{theorem}
\begin{proof}
By induction on the deduction sequence.
\emph{Base}: Initially, $\Pi(p,a) = V_a$; the invariant holds trivially.
\emph{Step}: \textsc{Apply} returns \textsc{ok} only if $v \in \Pi(p,a)$, ensuring consistency.
Arc-consistency propagation discards $v$ from $\Pi(p',a)$ for all $p'\neq p$, maintaining bijectivity.
Any empty domain triggers \textsc{contradiction}, so no inconsistent state is ever accepted.
\end{proof}

\begin{theorem}[Propagation completeness]
If $|\Pi(p,a)|=1$ for some $(p,a)$ after accepting a deduction, the propagator
derives this assignment automatically without an additional LLM call.
The cascade reaches an arc-consistency fixed point in $O(n^2 k)$ steps.
\end{theorem}
\begin{proof}
Each cascade step strictly reduces the total domain size $\sum_{p,a}|\Pi(p,a)|$, which is bounded below by $nk$; thus, the cascade terminates. The $O(n^2 k)$ bound follows from standard arc-consistency analysis~\cite{mackworth1977ac,hoffmann2001ff}.
\end{proof}

\noindent\textbf{Soundness caveat.}
The propagator guarantees \emph{consistency}, not \emph{soundness}.
A consistent-but-incorrect deduction (valid against the prior state but not
entailed by the clues) can still be accepted.
The residual failure rate on hard puzzles reflects exactly this gap.



\section{LGP-20 Benchmark Description}
\label{app:lgp20}

LGP-20 adds six verified puzzles to LGP-14. Solutions are available in
the released code. Difficulty mirrors LGP-14:
\textbf{E5-Colors} (3 entities, 2 attrs, 4 clues),
\textbf{E6-School} (3 entities, 2 attrs, 5 clues),
\textbf{M6-Travel} (4 entities, 2 attrs, 8 clues),
\textbf{M7-Food} (4 entities, 2 attrs, 7 clues),
\textbf{H6-Full} (4 entities, 3 attrs, 10 clues), and
\textbf{H7-5People} (5 entities, 2 attrs, 9 clues).

\section{Full Puzzle Descriptions}
\label{app:puzzles}

All fourteen LGP-14 puzzles are listed below with their complete clue
sets; solutions are in Table~\ref{tab:solutions}.

\paragraph{Easy puzzles.}
\textbf{E1-Pets}: Alice, Bob, Carol; $\{$Red, Blue, Green$\}\times\{$Cat, Dog, Fish$\}$;
5 clues including ``Bob lives in the blue house'' and
``The person in the red house has a cat.''
\textbf{E2-Jobs}: Adam, Beth, Chris; $\{$Tea, Coffee, Water$\}\times\{$Doctor, Teacher, Engineer$\}$;
5 clues including ``The doctor drinks coffee.''
\textbf{E3-Sport}: Lena, Milo, Nora; $\{$Tennis, Swim, Run$\}\times\{$Paris, London, Tokyo$\}$;
4 clues including ``The person in London swims'' and ``Lena is not in Tokyo.''
\textbf{E4-Music}: Jake, Kim, Leo; $\{$Piano, Guitar, Drums$\}\times\{$Mon, Tue, Wed$\}$;
4 clues including ``The piano player practices on Wednesday.''

\paragraph{Medium puzzles.}
\textbf{M1-Houses}: Alice, Bob, Carol, Dave; $4\times\{$color, drink$\}$; 6 clues.
\textbf{M2-Nationals}: Ana, Ben, Carl, Donna; $4\times\{$nationality, pet$\}$; 8 clues (disambiguated).
\textbf{M3-Hobby}: Ivan, Jana, Karl, Lea; $4\times\{$color, hobby$\}$; 6 clues.
\textbf{M4-PetFruit}: Omar, Petra, Quinn, Rosa; $4\times\{$pet, fruit$\}$; 7 clues.
\textbf{M5-Cities}: Alice, Bob, Carol, Dave; $4\times\{$city, sport$\}$; 7 clues.

\paragraph{Hard puzzles.}
\textbf{H1-Full}: Alex, Blake, Casey, Drew; $4\times\{$color, drink, pet$\}$; 9 clues.
\textbf{H2-Full}: Pam, Quinn, Rosa, Sam; $4\times\{$color, job, pet$\}$; 9 clues.
\textbf{H3-5People}: Alex, Beth, Carl, Dana, Eric; $5\times\{$color, drink$\}$; 8 clues.
\textbf{H4-Full}: Kim, Lee, Mia, Ned; $4\times\{$color, job, pet$\}$; 10 clues.
\textbf{H5-Full}: Ash, Bay, Cole, Drew; $4\times\{$color, drink, pet$\}$; 10 clues.

\begin{table}[h]
\centering\small
\setlength{\tabcolsep}{3pt}
\resizebox{\columnwidth}{!}{%
\begin{tabular}{lll}
\toprule
\textbf{Puzzle} & \textbf{Entities} & \textbf{Solution (Attr1 / Attr2 / Attr3)}\\
\midrule
E1 & Alice / Bob / Carol & Red+Cat / Blue+Dog / Green+Fish\\
E2 & Adam / Beth / Chris & Tea+Eng. / Coffee+Doctor / Water+Teacher\\
E3 & Lena / Milo / Nora & Tennis+Paris / Swim+London / Run+Tokyo\\
E4 & Jake / Kim / Leo   & Guitar+Mon / Drums+Tue / Piano+Wed\\
M1 & Alice / Bob / Carol / Dave & Green+Tea / Red+Juice / Blue+Coffee / Yellow+Water\\
M2 & Ana / Ben / Carl / Donna & English+Cat / German+Fish / Spanish+Bird / French+Dog\\
M3 & Ivan / Jana / Karl / Lea & Blue+Dance / Green+Read / White+Paint / Red+Cook\\
M4 & Omar / Petra / Quinn / Rosa & Cat+Cherry / Parrot+Date / Snake+Banana / Fish+Apple\\
M5 & Alice / Bob / Carol / Dave & London+Run / Berlin+Tennis / Paris+Swim / Rome+Soccer\\
H1 & Alex / Blake / Casey / Drew & Green+Tea+Bird / Red+Juice+Dog / Blue+Coffee+Fish / Yellow+Water+Cat\\
H2 & Pam / Quinn / Rosa / Sam & Red+Teacher+Fish / Blue+Chef+Bird / Yellow+Doctor+Cat / Green+Eng.+Dog\\
H3 & Alex / Beth / Carl / Dana / Eric & White+Tea / Blue+Milk / Red+Coffee / Green+Juice / Yellow+Water\\
H4 & Kim / Lee / Mia / Ned & Red+Doctor+Cat / Green+Lawyer+Bird / Yellow+Teacher+Fish / Blue+Chef+Dog\\
H5 & Ash / Bay / Cole / Drew & Red+Tea+Cat / Blue+Coffee+Bird / Green+Juice+Dog / Yellow+Water+Fish\\
\bottomrule
\end{tabular}}
\caption{Ground-truth solutions for all LGP-14 puzzles.}
\label{tab:solutions}
\end{table}

\section{System Prompts}
\label{app:prompts}

All methods receive the same \textbf{problem header}:

\begin{quote}\small
\texttt{LOGIC PUZZLE}\\
\texttt{People: \{people\}}\\
\texttt{Attributes: \{attr (vals)\}}\\
\texttt{Clues: \{numbered clue list\}}
\end{quote}

\textbf{Direct} appends: ``Provide your final answer as: [person]:
attr=value, ...''

\textbf{CoT} appends: ``Think step by step through the clues to
determine each person's attributes. Show your reasoning, then give the
final answer.''

\textbf{Self-Refine}: the first call is Direct; the second call appends the
first answer and asks: ``Carefully verify each clue against your
answer. List any violated clues, correct them, and provide the revised
final answer.''

\textbf{SymStep / SymStep+G} append after the problem header:

\begin{quote}\small
You are solving a logic problem by making one deduction at a time.\\
Each deduction must follow EXACTLY one of these formats:\\
\quad\texttt{DEDUCE: <Person>, <attribute>, <Value>}\\
\quad\texttt{DEDUCE: <Person>, <attribute>, NOT <Value>}\\
After each deduction, you will receive verification feedback.\\
When the problem is fully solved, output: \texttt{CONCLUDE: done}\\
Do not output multiple deductions at once.
\end{quote}

Verifier feedback:
\begin{itemize}
  \item \textbf{Accepted}: \texttt{[Verifier]: \checkmark\ Correct
        deduction accepted.}
        (followed by an MRV hint if SymStep+G is used: \texttt{[Hint] X's attr must be one of: \{...\}})
  \item \textbf{Contradiction}: \texttt{[Verifier]: \textbf{\texttimes}\ CONTRADICTION --- \{exact message\}. Please reconsider.}
\end{itemize}

\section{Example SymStep+G Transcript}
\label{app:transcript}

The following is an abridged transcript of SymStep+G on puzzle E1-Pets
(solution: Alice=Red+Cat, Bob=Blue+Dog, Carol=Green+Fish).

\begin{quote}\small
\textbf{LLM:} \texttt{DEDUCE: Bob, color, Blue}\\
\textbf{Verifier:} \checkmark\ Verified. [Hint] Alice's color must be one of: \{Red, Green\}\\[2pt]
\textbf{LLM:} \texttt{DEDUCE: Carol, color, Green}\\
\textbf{Verifier:} \checkmark\ Verified. [Hint] Alice's color must be one of: \{Red\}\\[2pt]
\textbf{LLM:} \texttt{DEDUCE: Alice, color, Red}\\
\textbf{Verifier:} \checkmark\ Verified. [Hint] Alice's pet must be one of: \{Cat\}\\[2pt]
\textbf{LLM:} \texttt{DEDUCE: Alice, pet, Cat}\\
\textbf{Verifier:} \checkmark\ Verified. [Hint] Bob's pet must be one of: \{Dog, Fish\}\\[2pt]
\textbf{LLM:} \texttt{DEDUCE: Carol, pet, Fish}\\
\textbf{Verifier:} \checkmark\ Verified. [Problem solved!] Output CONCLUDE: done\\[2pt]
\textbf{LLM:} \texttt{CONCLUDE: done}
\end{quote}

After Step 2, the arc-consistency cascade automatically narrows Alice's
color to \{Red\}; the LLM never explicitly deduced this, yet the hint
reveals it. This cascade efficiency is the key benefit of unit propagation:
$\approx$60\% of assignments are derived automatically, without requiring
separate LLM calls.

\renewcommand{\refname}{Full Bibliography}
\bibliographystyle{named}
\bibliography{ijcai26}

\end{document}